%% file: main.tex
\PassOptionsToPackage{table}{xcolor}
\documentclass{article} % For LaTeX2e
\usepackage[T1]{fontenc} % Keep the template Times fonts across TeX engines.
\usepackage{iclr2027_conference,times}

\input{math_commands.tex}

\usepackage{amsthm}
\newtheorem{theorem}{Theorem}[section]
\newtheorem{proposition}[theorem]{Proposition}

\usepackage{hyperref}
\usepackage{url}
\usepackage{enumitem}
\usepackage{graphicx}
\usepackage{booktabs}
\usepackage{multirow}
\usepackage{makecell}
\usepackage{algorithm}
\usepackage{algpseudocode}
\usepackage{pifont}
\usepackage{threeparttable}
\usepackage{subcaption}
\usepackage[table]{xcolor}

\usepackage{listings}
\usepackage{tabularx}
\definecolor{promptblue}{RGB}{20,80,180}

\definecolor{blue}{HTML}{003366}
\definecolor{gold}{HTML}{996600}

\usepackage{wrapfig}

\lstdefinestyle{jsoncompact}{
  basicstyle=\ttfamily\footnotesize,
  breaklines=true,
  breakatwhitespace=false,
  columns=fullflexible,
  keepspaces=true,
  showstringspaces=false,
  frame=single,
  framerule=0.2pt,
  xleftmargin=0.5em,
  xrightmargin=0.5em,
  escapeinside={(*@}{@*)}
}

\DeclareRobustCommand{\name}{DoubleFree}

\title{
  \name{}: Exploiting Exploration History for Rollout-Free Multi-Task LLMs Training
}

\author{\resizebox{\dimexpr\textwidth-2\tabcolsep\relax}{!}{Yingping Wang$^{1}$, Min Cai$^{2}$, Chenyu Yuan$^{1}$, Xi Ye$^{2}$, Zhijiang Guo$^{1}$, Jiayi Huang$^{1,*}$}\\[0.5em]
$^{1}$The Hong Kong University of Science and Technology (Guangzhou)\\
$^{2}$University of Alberta\\
{\texttt{hjy@hkust-gz.edu.cn}}}

\iclrfinalcopy
\begin{document}

\maketitle
\begingroup
\renewcommand{\thefootnote}{\fnsymbol{footnote}}
\footnotetext[1]{Corresponding author}
\endgroup
% The style sets an ICLR publication header inside \maketitle.
% Clear it for the preprint without changing the conference style file.
\fancyhead{}

\begin{abstract}
Multi-task post-training can develop domain specialists through independent
reinforcement learning (RL) and consolidate their capabilities through
multi-teacher on-policy distillation (MOPD).
However, standard MOPD generates on-policy student responses and keeps
multiple teachers online throughout consolidation, introducing substantial
system costs.
We view this pipeline as domain-wise exploration followed by cross-domain
exploitation.
Domain RL produces not only the final specialists, but also the trajectories
collected throughout their training.
Existing methods primarily use the former while leaving the latter underutilized.
We introduce \name{}, which reuses these historical trajectories for
rollout-free consolidation and precomputes teacher supervision for
online-teacher-free student training.
To understand why this works, we analyze when updates from historical
trajectories can approximate on-policy distillation updates and how combining
trajectories across exploration stages can better support an evolving student.
\name{} requires only minimal changes to existing post-training frameworks.
Across math, code, and instruction following on Qwen3-4B, Qwen3-8B, and
SmolLM3-3B, \name{} achieves $2.99$-$5.75\times$ consolidation speedups
and reduces GPU-hours by $80.9$-$90.1\%$, while maintaining comparable
multi-domain performance.
\end{abstract}

\input{Sections/00_intro}
\input{Sections/01_problem}
\input{Sections/02_method}
\input{Sections/03_theory}
\input{Sections/04_eval}
\input{Sections/05_relat}
\input{Sections/06_conclu}
% Flush main-text tables before references.
\clearpage
\bibliographystyle{iclr2027_conference}
\bibliography{ref}

\newpage
\appendix
\onecolumn
\input{Sections/08_implementation_appendix}
\input{Sections/09_theory_appendix}
\input{Sections/10_results_appendix}

\end{document}

%% file: math_commands.tex
\usepackage{amsmath,amsfonts,bm}

\def\eqref#1{equation~\ref{#1}}
\def\1{\bm{1}}

\DeclareMathAlphabet{\mathsfit}{\encodingdefault}{\sfdefault}{m}{sl}
\SetMathAlphabet{\mathsfit}{bold}{\encodingdefault}{\sfdefault}{bx}{n}

%% file: Sections/00_intro.tex
\section{Introduction}
\label{sec:intro}

\begin{wrapfigure}{r}{0.46\textwidth}
    \centering
    \includegraphics[width=\linewidth]{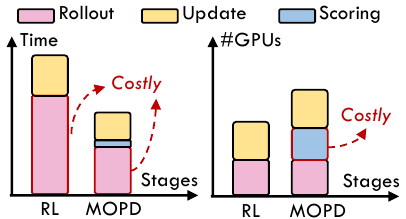}
    \caption{Two system costs in RL-then-MOPD:
    repeated response generation and GPU resources reserved
    for online multi-teacher serving.}
    \label{fig:teaser}
\end{wrapfigure}

Multi-task LLM post-training seeks to combine capabilities such as
mathematical reasoning, code generation, and instruction following
in a single model~\citep{lambert2025tulu3,qwen2025qwen3}.
An increasingly studied approach first develops domain specialists
through independent reinforcement learning (RL), then consolidates
their capabilities through multi-teacher on-policy distillation
(MOPD)~\citep{mopd,interns2}.

This RL-then-MOPD pipeline incurs two major system costs
(Figure~\ref{fig:teaser}).
First, rollout is a major cost in both stages~\cite{rollpacker}, and accounts for about $70\%$ of per-batch time during consolidation.
Yet consolidation generates new student responses despite the domain-specific experience already captured in RL trajectories.
Second, student updates depend on online teacher scoring,
typically requiring dedicated teacher GPUs throughout training.
Specifically, teachers reserve $42.9\%$ of GPUs, while remaining idle for over $90\%$ of their allocated time (Section~\ref{sec:problem}).

We revisit this pipeline through an \emph{exploration-exploitation}
view (Section~\ref{sec:formu}): domain-wise exploration develops
specialist capabilities, while shared exploitation transfers
these capabilities into a single policy.
This view reveals two outputs of domain-wise exploration available for shared exploitation: \textbf{final specialist policies} and \textbf{historical response trajectories}.
Existing distillation methods exploit the former for supervision but largely discard the latter, instead generating additional responses from the current student, initial student, or final specialists~\citep{mopd,wu2026lightningopd,seqkd}.
This presents the opportunity to exploit both outputs to improve
the consolidation efficiency.

We introduce \name{}, which exploits both outputs of domain-wise exploration: it reuses historical trajectories for \textbf{rollout-free} consolidation and pre-scores them with final specialists for \textbf{online-teacher-free} student training.
For policy distillation, the student learns from cached responses
and teacher scores while recomputing its own log-probabilities.
Together, trajectory reuse and teacher pre-scoring eliminate
additional response generation and release teacher GPUs
before student training.
A key question is whether exploration trajectories can support
an evolving shared student.
Collected across different RL substages, these trajectories may jointly
provide broader coverage of student behavior than responses
from either the initial policy or final specialists alone.
In Section~\ref{sec:theory}, we bound the discrepancy between
on-policy and reuse distillation updates in terms of distribution
mismatch and characterize potential coverage benefits of
historical policy mixtures.

Our experiments are conducted on NVIDIA H100-80GB GPUs, with \name{} implemented in verl~\citep{verl}.
Across math, code, and instruction following on Qwen3-4B,
Qwen3-8B, and SmolLM3-3B, \name{} reduces consolidation
GPU-hours by $80.9$-$90.1\%$ and achieves $2.99$-$5.75\times$
speedups over Standard MOPD across \name{}'s SFT and MPD variants,
while retaining comparable multi-domain performance.
Our contributions are as follows:
\begin{itemize}[leftmargin=1.5em, itemsep=0pt, topsep=2pt, parsep=0pt]

\item We formulate RL-then-MOPD as an explore-exploit framework
and identify redundant response generation and persistent
teacher GPU allocation as two system bottlenecks.

\item \name{} combines exploration trajectory reuse with teacher
pre-scoring for rollout-free consolidation and online-teacher-free
student training. Our analysis further characterizes when such reuse
can approximate on-policy updates and how trajectories across
exploration stages can complement each other.

\item Experiments across three models of different scales and families
cover nine benchmarks in math, code, and instruction following.
\name{} substantially reduces consolidation time and GPU-hours
while maintaining comparable multi-domain accuracy.

\end{itemize}

%% file: Sections/01_problem.tex
\section{MOPD System Bottlenecks}
\label{sec:problem}

\begin{figure}[t]
    \centering
    \includegraphics[width=0.9\linewidth]{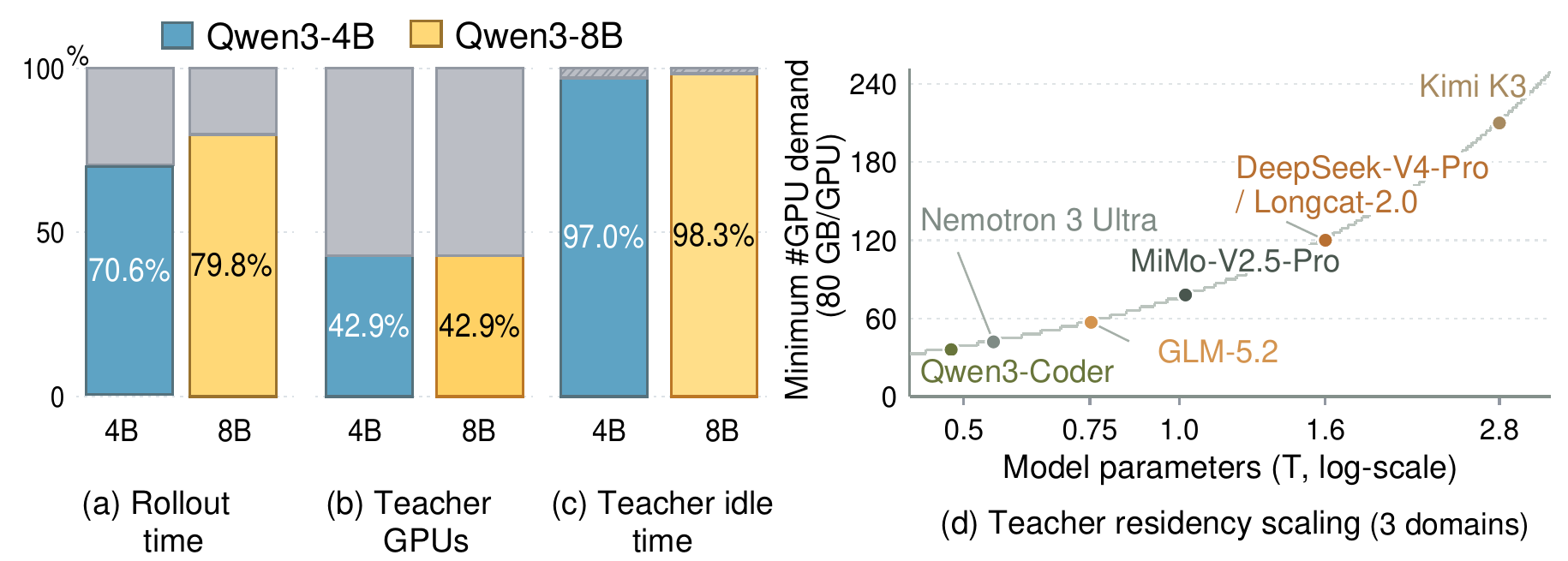}
    \caption{MOPD system costs.
    (a-c) Rollout time, teacher GPU allocation, and teacher idle time
    under two settings.
    (d) Minimum GPU demand for three resident domain teachers
    (BF16 weights). }
    \label{fig:mopd_settings}
\end{figure}

We profile Standard MOPD on Qwen3-series across three domains, with separate GPU pools for student training and teacher serving.
Student/teacher GPU allocations are 4/3 for 4B and 8/6 for 8B model.
Each mixed-domain batch is globally synchronized in our setting.
Figure~\ref{fig:mopd_settings} summarizes rollout time, teacher
allocation and idle time, and teacher GPU demand as model size increases.

\paragraph{Costly rollout across stages.}
Standard MOPD generates fresh responses from the current student
for every update.
As shown in Figure~\ref{fig:mopd_settings}(a), rollout accounts for
$70.6\%$ of per-batch time on Qwen3-4B
($55.02$ of $77.95$ seconds) and $79.8\%$ on Qwen3-8B
($76.06$ of $95.27$ seconds).
This generation cost is paid again during consolidation even though
domain RL has already produced substantial trajectories through
domain-specific exploration.
Offline alternatives move response generation by the initial student
or final specialists ahead of training, but do not eliminate this
additional cost.
Such redundant generation is particularly costly for long-response
reasoning workloads.
These observations raise a key research question:
\emph{Can existing exploration trajectories replace on-policy
rollouts during consolidation while preserving distillation quality?}

\paragraph{Intermittent multi-teacher serving.}
Online teachers provide supervision for trajectories from their
respective domains, and each student update proceeds only after
scoring is complete across all domains.
In a conventional resident-teacher deployment, domain specialists
occupy dedicated GPUs throughout training, accounting for $42.9\%$
of the total GPU allocation in both profiled settings.
Yet request-level timings indicate that over $97\%$ of allocated
teacher GPU-time falls outside scoring calls in both settings
(Figure~\ref{fig:mopd_settings}(b-c)).
Moreover, the resource requirement of teacher residency scales
with model size.
As shown in Figure~\ref{fig:mopd_settings}(d), three domain teachers
require at least $120$ 80\,GB GPUs at $1.6$T parameters per teacher
and $210$ at $2.8$T for their BF16 weights alone.
This raises a complementary question:
\emph{Can teacher supervision be delivered without continuously
reserving substantial GPU resources for intermittent scoring?}

%% file: Sections/02_method.tex
\section{Method: DoubleFree}
\label{sec:method}

\name{} addresses the two bottlenecks by exploiting both outputs of
domain-wise exploration: historical trajectories are reused for
consolidation, while final specialists provide precomputed supervision
(Figure~\ref{fig:overview}).
Trajectory reuse eliminates additional response generation, while
fixed responses allow specialist supervision to be computed offline.
Student training is therefore \emph{rollout-free} and
\emph{online-teacher-free}.

\begin{figure}[htbp]
    \centering
    \includegraphics[width=\linewidth]{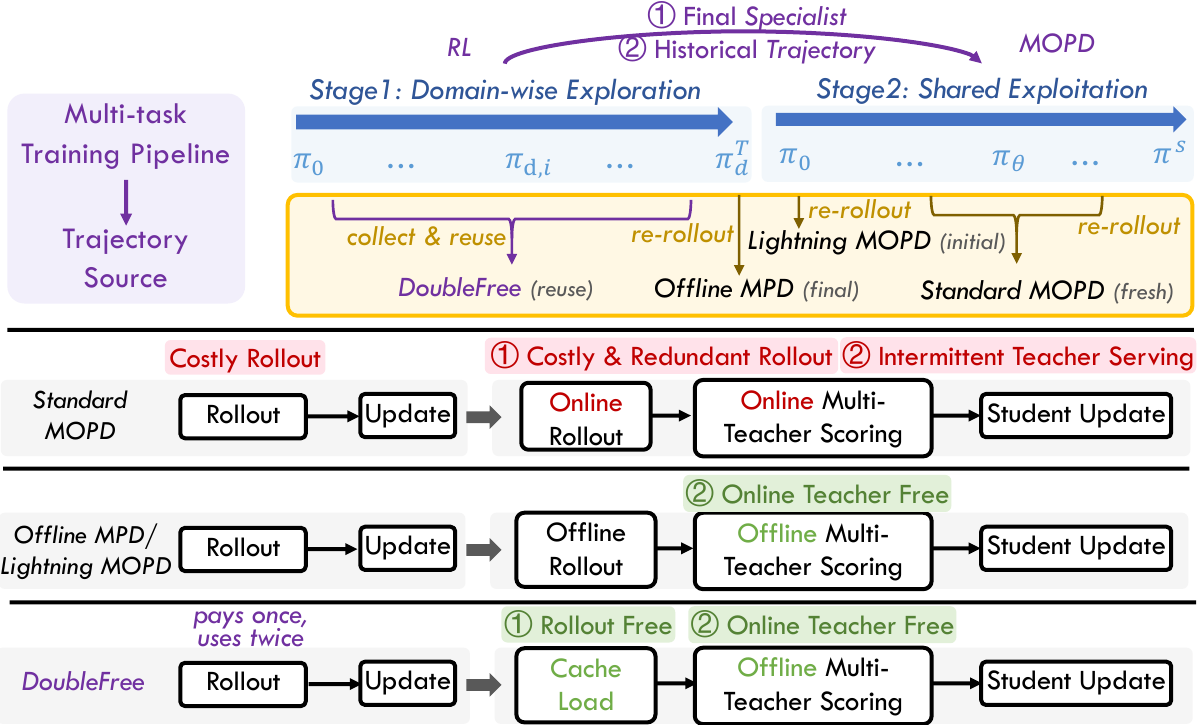}
    \caption{Overview of DoubleFree.
    Exploration stage produces both specialist policies and historical trajectories.
    \name{} reuses the trajectories and precomputes specialist scores,
    removing additional generation and online teacher residency from
    student training.}
    \label{fig:overview}
\end{figure}

\subsection{Explore-Exploit Formulation}
\label{sec:formu}
Let $\pi_0$ be an initial shared policy and
$\{\mathcal D_d,r_d\}_{d=1}^{M}$ a collection of domains, with
prompt distribution $\mathcal D_d$ and verifiable reward $r_d$.
The goal is a shared policy with strong aggregate performance,
\begin{equation}
    \max_\pi J(\pi)=\sum_{d=1}^{M}w_d J_d(\pi),
    \qquad
    J_d(\pi)=\mathbb E_{x\sim\mathcal D_d,\;y\sim\pi(\cdot\mid x)}
    [r_d(x,y)],
    \label{eq:multitask_objective}
\end{equation}
where $w_d\ge0$ and $\sum_d w_d=1$.
We organize RL-then-MOPD into domain-wise exploration and shared
exploitation toward the aggregate objective in
Eq.~\ref{eq:multitask_objective}.

\paragraph{Domain-wise exploration.}
Each explorer is initialized from $\pi_0$ and independently optimized
on $J_d$, producing two outputs: a final specialist $\pi_d^T$ and a
historical trajectory buffer $\mathcal B_d$ of prompt-response pairs
$\tau=(x,y)$ collected throughout RL.
We denote the policies that generated these responses
by $\{\pi_{d,i}\}_{i\in\mathcal I_d}$, where $i$ indexes RL
substages within exploration. 

\paragraph{Shared exploitation.}
A student $\pi_\theta$, initialized from $\pi_0$, integrates the
specialists through multi-domain policy distillation (MPD). For a response
prefix $h_t=(x,y_{<t})$, the update signal is
\begin{equation}
    \log\pi_d^T(y_t\mid h_t)-\log\pi_\theta(y_t\mid h_t).
    \label{eq:distillation_signal}
\end{equation}
The teacher defines the supervision, while the trajectory source
determines the prefixes and tokens on which it acts. 
% We study one exploration phase followed by one exploitation phase, holding the
% final specialists and available RL data fixed across comparisons.

\subsection{Trajectory Reuse for Rollout-Free Consolidation}
\label{sec:trajectory_sources}

\name{} loads complete responses from $\mathcal B_d$ rather than
generating new student rollouts.
The exploration buffer contains responses generated throughout
domain-specific capability development, spanning policies from different RL substages.
With prompts drawn from the same domain
distribution $\mathcal D_d$, the underlying response distribution (\textit{reuse}) is
\begin{equation}
    p_{\mathrm{reuse},d}(y\mid x)=
    \sum_{i\in\mathcal I_d}\alpha_{d,i}\pi_{d,i}(y\mid x),
    \label{eq:reuse_distribution}
\end{equation}
where $\alpha_{d,i}\ge0$ is the mixture weight of historical
policy $\pi_{d,i}$ and $\sum_i\alpha_{d,i}=1$.

Within the framework of Section~\ref{sec:formu},
we express alternative trajectory sources as
\begin{align}
    p_{\mathrm{fresh},d}(y\mid x;\theta)&=\pi_\theta(y\mid x),
    \label{eq:fresh_distribution}\\
    p_{\mathrm{initial},d}(y\mid x)&=\pi_0(y\mid x),
    \label{eq:initial_distribution}\\
    p_{\mathrm{final},d}(y\mid x)&=\pi_d^T(y\mid x).
    \label{eq:final_distribution}
\end{align}
We extend these trajectory sources to multi-domain consolidation, yielding Standard MOPD (\emph{fresh}), Lightning MOPD (\emph{initial}), Offline MPD (\emph{final}), and \name{} (\emph{reuse}), all under the same final-specialist supervision.
Here, reused responses need not be positive demonstrations: the
teacher-student signal can increase or decrease the probability of
their observed tokens. Exploration history supplies contexts visited
during capability development, while the final specialist determines
how the student should learn from them. 
This changes the sampling distribution rather than reproducing on-policy distillation exactly.

\subsection{Pre-scoring: Online-Teacher-Free Training}
\label{sec:teacher_free}

For fixed responses, scores from frozen teachers are independent of the evolving student. We therefore precompute and cache observed-token log-probabilities under teacher forcing,
\begin{equation}
    \mathbf q_d^T(\tau)=
    \bigl(\log\pi_d^T(y_t\mid h_t)\bigr)_{t=1}^{|y|}.
    \label{eq:teacher_scores}
\end{equation}
During student training, cached teacher scores remain fixed while student log-probabilities are recomputed on the reused responses.
Pre-scoring separates teacher and student GPU allocations in time, removing online teacher queries and scoring barriers.
Teachers can also be loaded and released sequentially during preparation, further reducing peak GPU demand.

Exploration typically generates more responses per prompt than
distillation uses, enabling optional trajectory selection.
We construct the scored training buffer as
\begin{equation}
    \widetilde{\mathcal B}_d=
    \operatorname{Select}_{\phi}
    \left\{(\tau,\mathbf q_d^T(\tau)):\tau\in\mathcal B_d\right\}.
    \label{eq:support_selection}
\end{equation}
Here, $\phi$ may retain all responses or select a subset,
for example by random sampling.

\subsection{Shared-Student Optimization}

Our primary instantiation, \name{}-MPD, optimizes the shared student
using final-specialist supervision on reused trajectories.
We follow a PPO-style clipped objective: cached teacher
log-probabilities and recomputed student log-probabilities define
clipped, detached token-level distillation advantages.
Updates are anchored to the pre-update student policy, and the loss
is averaged over valid response tokens.
The full objective is given in Appendix~\ref{app:student_objective}.
Trajectory reuse is also compatible with SFT.
We include \name{}-SFT as a variant that directly minimizes
response-token cross-entropy on reused trajectories, without teacher
scoring.

%% file: Sections/03_theory.tex
\section{Theoretical Analysis of Trajectory Reuse}
\label{sec:theory}

\subsection{Bounded Discrepancy from On-Policy Updates}
\label{sec:update_discrepancy}

The key question is how replacing fresh student rollouts with
historical exploration trajectories changes the distillation update.
To characterize this discrepancy, we fix the student parameters
$\theta$ and final specialists $\pi_d^T$.
In each domain $d$, \emph{fresh} and \emph{reuse} draw prompts from
the same $\mathcal D_d$ but generally induce different response
distributions.
Following the local-surrogate perspective of
\citet{wu2026lightningopd}, we compare unnormalized population
update directions at the PPO student anchor.

Let $f_d(x,y;\theta)$ be a response's update contribution in domain $d$ and
$g_b(\theta)$ the expected update under source
$b\in\{\mathrm{fresh},\mathrm{reuse}\}$:
\begin{equation}
\begin{aligned}
    f_d(x,y;\theta)
    &= \sum_{t=1}^{|y|}
    A_{d,t}(x,y;\theta)\nabla_\theta
    \log\pi_\theta(y_t\mid h_t),\\
    g_b(\theta)
    &= \sum_d w_d
    \mathbb E_{x\sim\mathcal D_d,\,
    y\sim p_{b,d}(\cdot\mid x)}
    [f_d(x,y;\theta)].
\end{aligned}
\label{eq:exploitation_updates}
\end{equation}
Here, $|y|$ is the response length, $y_t$ its $t$-th token,
and $h_t=(x,y_{<t})$ its prefix.
The advantage $A_{d,t}$ is the teacher-student log-probability
difference in Eq.~\ref{eq:distillation_signal}, clipped to
$[-c,c]$ for $c>0$ and held constant during differentiation.
The domain weights satisfy $w_d\ge0$ and $\sum_d w_d=1$.

\begin{theorem}[Bounded update discrepancy]
\label{thm:reuse_discrepancy}
At a fixed student state $\theta$, assume for each domain $d$ that
$p_{\mathrm{fresh},d}(\cdot\mid x;\theta)
\ll p_{\mathrm{reuse},d}(\cdot\mid x)$
for every prompt $x$ with $\mathcal D_d(x)>0$.
(Here, $\ll$ denotes absolute continuity: every response with
positive fresh probability also has positive reuse probability.)
Let $B_d\ge0$ be a finite bound on the root-mean-square
update contribution,
\[
    \mathbb E_{x\sim\mathcal D_d,\,
    y\sim p_{\mathrm{reuse},d}(\cdot\mid x)}
    \|f_d(x,y;\theta)\|_2^2
    \le B_d^2,
\]
where $\|\cdot\|_2$ is the Euclidean norm.
Assume the nonnegative distribution mismatch $\delta_d(\theta)$,
defined by
\begin{equation}
    \delta_d^2(\theta)
    =
    \mathbb E_{x\sim\mathcal D_d}
    \chi^2\!\left(
    p_{\mathrm{fresh},d}(\cdot\mid x;\theta)
    \,\middle\|\,
    p_{\mathrm{reuse},d}(\cdot\mid x)
    \right),
    \label{eq:reuse_mismatch}
\end{equation}
is finite. For response distributions $p\ll q$,
$\chi^2(p\|q)=\sum_{y:q(y)>0}(p(y)-q(y))^2/q(y)$.
Then
\begin{equation}
    \|g_{\mathrm{reuse}}(\theta)
    -g_{\mathrm{fresh}}(\theta)\|_2
    \le
    \sum_d w_d B_d\delta_d(\theta)
    \equiv \Delta(\theta),
    \label{eq:reuse_discrepancy_bound}
\end{equation}
where $\Delta(\theta)$ is the aggregate update-discrepancy bound.
\end{theorem}
Theorem~\ref{thm:reuse_discrepancy} shows that reuse remains close
to the fresh update when distribution mismatch and update magnitudes
are controlled (proof in Appendix~\ref{app:discrepancy_proof}).
Importantly, the direction of the divergence penalizes responses
that are likely under the current student but unlikely under
\emph{reuse}.
The discrepancy therefore depends on \emph{how well the historical source covers the responses the current student would generate}.

\subsection{Coverage Benefits of Historical Policy Mixtures}
\label{sec:history_coverage}

Historical reuse combines trajectories generated by policies
$\{\pi_{d,i}\}$ across exploration substages.
This mixture can inherit coverage from any retained historical policy, as formalized below.

\begin{proposition}[Coverage inherited from history]
\label{prop:history_coverage}
For each domain $d$, let $\pi_{d,i}$ be a retained historical policy
with mixture weight $\alpha_{d,i}>0$ in Eq.~\ref{eq:reuse_distribution}.
The distribution mismatch defined in Eq.~\ref{eq:reuse_mismatch}
is bounded as follows:
\begin{equation}
    1+\delta_d^2(\theta)
    \le\frac{1+\mathbb E_{x\sim\mathcal D_d}
    \chi^2\!\left(\pi_\theta(\cdot\mid x)
    \,\middle\|\,\pi_{d,i}(\cdot\mid x)\right)}{\alpha_{d,i}}.
    \label{eq:history_coverage_bound}
\end{equation}
\end{proposition}

Because \emph{reuse} assigns each response at least an
$\alpha_{d,i}$ fraction of its probability under $\pi_{d,i}$,
good coverage by any sufficiently weighted historical policy
transfers to the reuse mixture (Appendix~\ref{app:coverage_proof}).
More importantly, historical policies may provide complementary
coverage.
Responses rarely generated by either endpoint may be more common
at intermediate substages, for example due to differences in reward, teacher-student gap, or response length.
Their mixture can therefore cover student behavior that no single
historical policy covers well (Appendix~\ref{app:coverage_example}).

\begin{center}
\begingroup
\setlength{\fboxsep}{7pt}
\setlength{\fboxrule}{0.3pt}
\fcolorbox{black!25}{black!3}{%
\begin{minipage}{\dimexpr\linewidth-2\fboxsep-2\fboxrule\relax}
\textbf{Takeaway.}
Reusing exploration trajectories stays close to on-policy
distillation as long as the history covers what the student
would generate. 
Historical mixtures are naturally suited to
provide this coverage, as they combine complementary responses
from every exploration stage.
\end{minipage}%
}
\endgroup
\end{center}

%% file: Sections/04_eval.tex
\section{Experiments}
\label{sec:exp}

\subsection{Experimental Setup}
\label{subsec:setup}

\paragraph{Models and datasets.}
We evaluate Qwen3-4B and Qwen3-8B~\citep{qwen2025qwen3}, as well as
SmolLM3-3B~\citep{bakouch2025smollm3}, in non-thinking mode
across math, code, and instruction following. For domain RL, we sample
17K examples each from DAPO-Math-17k~\citep{dapo2025},
DeepCoder-Preview-Dataset~\citep{deepcoder2025}, and
Nemotron-RL-instruction\_following~\citep{nvidia2025nemotronif}, respectively.
Evaluation spans nine benchmarks across the three domains,
with benchmark details and protocols in
Appendix~\ref{app:benchmark_protocol}.

\paragraph{Baselines.}
We compare \name{}-MPD with Standard MOPD,
Offline MPD, and Lightning MOPD, which use different trajectory
sources under the same final-specialist supervision.
Mixed SFT and \name{}-SFT provide response-token supervision baselines.
We focus comprehensive
baseline comparisons and ablations on Qwen3-4B.
On Qwen3-8B and SmolLM3-3B, we evaluate both \name{} variants
against Standard MOPD to assess cross-model generalization.
For the main accuracy comparisons, we evaluate \name{}-MPD over
three random seeds on each model, while using the same fixed seed
for all other methods.
Table~\ref{tab:consolidation_config} summarizes the method configurations.

\paragraph{Training protocol.}
Within each model, methods share the student initialization and use
equal numbers of training prompts. We implement all
methods in verl~\citep{verl} with vLLM~\citep{vllm2023} on a testbed
with 16 NVIDIA H100-80GB GPUs.
For RL, actor training and rollout generation are colocated on the same GPU group.
Standard MOPD uses separate student and teacher GPU pools.
For example, four GPUs support student training and rollouts,
while each of the three domain teachers occupies one dedicated GPU.
More details in Appendix~\ref{app:training_details}.

\begin{table*}[t]
\centering
\caption{Comparison of multi-task training methods.}
\label{tab:consolidation_config}

\small
\setlength{\tabcolsep}{4pt}
\renewcommand{\arraystretch}{1.2}

\resizebox{0.9\linewidth}{!}{%
\begin{tabular}{@{}llcccl@{}}
\toprule
\textbf{Method}
& \shortstack{\textbf{Trajectory}\\\textbf{Source}}
& \shortstack{\textbf{Additional}\\\textbf{Rollout}}
& \shortstack{\textbf{Consolidation}\\\textbf{objective}}
& \shortstack{\textbf{Online}\\\textbf{Teacher}}
& \shortstack{\textbf{Supervision}\\\textbf{Source}} \\
\midrule

Standard MOPD
& Current student
& Online
& PD
& Yes
& Final specialist \\

Mixed SFT
& Final specialist
& Offline
& SFT
& No
& Generated response tokens \\

Offline MPD
& Final specialist
& Offline
& PD
& No
& Final specialist \\

Lightning MOPD
& Initial student
& Offline
& PD
& No
& Final specialist \\

\midrule

\name{}-SFT
& RL training policies
& None
& SFT
& No
& RL response tokens \\

\textbf{\name{}-MPD}
& RL training policies
& None
& PD
& No
& Final specialist \\

\bottomrule
\end{tabular}%
}
\end{table*}

\paragraph{Metrics.}
\emph{GPU-hours} measure allocated resource cost, summing
each phase's duration multiplied by its allocated GPU count.
\emph{Time} reports consolidation duration in hours, combining
measured training time and offline preparation time.
\emph{Speedup} is Standard MOPD's Time divided by the method's \emph{Time} for the same model.
For accuracy preservation, we report math accuracy, code pass@1, and strict prompt-level accuracy for instruction following. 
AVG equally weights within-domain benchmarks, and All AVG across the three domains.

\subsection{Training Efficiency}
\label{subsec:eff}

\begin{table*}[htbp]
\centering
\caption{Training efficiency across three models.
Teacher costs include offline pre-scoring or online teacher
residency, including idle time.
Cache load for offline methods includes cached trajectory access and associated batch-processing
overhead.}
\label{tab:mopd_efficiency}
\small
\setlength{\tabcolsep}{4pt}
\renewcommand{\arraystretch}{1.15}
\resizebox{0.9\linewidth}{!}{%
\begin{tabular}{@{}l*{9}{r}@{}}
\toprule
\multirow{2}{*}{\textbf{Method}}
& \multirow{2}{*}{\textbf{GPUs}}
& \multicolumn{6}{c}{\textbf{GPU-hours} $\downarrow$}
& \multirow{2}{*}{\shortstack{\textbf{Time}\\(h) $\downarrow$}}
& \multirow{2}{*}{\shortstack{\textbf{Speedup}\\$\uparrow$}} \\
\cmidrule(lr){3-8}
& & \shortstack{Offline/Online\\Rollout}
& \shortstack{Offline/Online\\Teacher}
& Cache load & Update & Others & \textbf{Total}
& & \\
\midrule
\rowcolor{gray!12}
\multicolumn{10}{c}{\textit{\textbf{Qwen3-4B}}} \\
\midrule
Math RL
& 8 & 51.37 & -- & -- & 75.70 & 2.05
& 129.13 & 16.14 & -- \\
Code RL
& 8 & 34.62 & -- & -- & 35.05 & 2.21
& 71.88 & 8.98 & -- \\
IF RL
& 8 & 17.35 & -- & -- & 6.31 & 1.95
& 25.61 & 3.20 & -- \\
\midrule
Mixed SFT
& 4 & 24.10 & -- & 0.63 & 10.56 & 2.31
& 37.60 & 9.44 & $1.40\times$ \\
Offline MPD
& 4 & 24.10 & 1.73 & 0.60 & 13.36 & 2.34
& 42.13 & 10.53 & $1.25\times$ \\
Lightning MOPD
& 4 & 9.22 & 1.34 & 0.68 & 6.14 & 2.47
& 19.85 & 4.96 & $2.66\times$ \\
Standard MOPD
& 4+3 & 36.68 & 39.56 & -- & 13.49 & 2.58
& 92.31 & 13.19 & $1.00\times$ \\
\midrule
\name{}-SFT
& 4 & -- & -- & 0.62 & 10.20 & 2.17
& 12.99 & 3.25 & $4.06\times$ \\
\name{}-MPD
& 4 & -- & 1.34 & 1.01 & 12.64 & 2.63
& 17.61 & 4.40 & $2.99\times$ \\
\midrule
\rowcolor{gray!12}
\multicolumn{10}{c}{\textit{\textbf{Qwen3-8B}}} \\
\midrule
Standard MOPD
& 8+6 & 101.41 & 96.71 & -- & 19.63 & 7.90
& 225.65 & 16.12 & $1.00\times$ \\
\midrule
\name{}-SFT
& 8 & -- & -- & 1.54 & 14.24 & 6.65
& 22.43 & 2.80 & $5.75\times$ \\
\name{}-MPD
& 8 & -- & 1.57 & 2.40 & 16.88 & 7.65
& 28.50 & 3.56 & $4.52\times$ \\
\midrule
\rowcolor{gray!12}
\multicolumn{10}{c}{\textit{\textbf{SmolLM3-3B}}} \\
\midrule
Standard MOPD
& 4+3 & 28.16 & 26.39 & -- & 5.57 & 1.45
& 61.58 & 8.80 & $1.00\times$ \\
\midrule
\name{}-SFT
& 4 & -- & -- & 0.53 & 4.25 & 1.78
& 6.56 & 1.67 & $5.26\times$ \\
\name{}-MPD
& 4 & -- & 1.19 & 0.77 & 5.40 & 2.33
& 9.70 & 2.43 & $3.63\times$ \\
\bottomrule
\end{tabular}%
}
\end{table*}

Table~\ref{tab:mopd_efficiency} shows that \name{} reduces consolidation
GPU-hours by $80.9$-$90.1\%$ and achieves $2.99$-$5.75\times$ speedups
across three models and both variants.
For example, on Qwen3-4B, \name{}-MPD reduces GPU-hours from
$92.31$ to $17.61$ and consolidation time from $13.19$ to $4.40$
hours.
Resource savings exceed time savings because \name{} both
shortens consolidation and eliminates dedicated teacher GPU
reservations, including idle periods.
The benefits extend to Qwen3-8B and SmolLM3-3B, where the MPD
variant achieves $4.52\times$ and $3.63\times$ speedups, respectively.

The cost breakdown for Qwen3-4B illustrates where these savings arise.
Domain RL already incurs $103.34$ rollout GPU-hours, yet
Standard MOPD adds $36.68$ student rollout GPU-hours and
$39.56$ teacher-residency GPU-hours, including idle time.
Offline MPD and Lightning MOPD avoid online teacher residency
but still require $24.10$ and $9.22$ generation GPU-hours,
respectively.
Here, Lightning benefits from shorter initial-student responses.
\name{} eliminates this additional overhead, requiring only
offline scoring for MPD and no teacher scoring for SFT.
Its higher cache-loading overhead reflects loading and validating
full candidate shards before optional selection, but remains small relative to the generation and residency costs avoided.

\subsection{Multi-Task Accuracy}
\label{subsec:acc}

\begin{table*}[htbp]
\centering
\caption{Multi-task accuracy (\%).
Bold marks the best consolidated student per model.
\name{}-MPD reports mean $\pm$ sample standard deviation over
three seeds.}
\label{tab:mopd_main}

\footnotesize
\setlength{\tabcolsep}{3.5pt}
\renewcommand{\arraystretch}{1.15}

\resizebox{\linewidth}{!}{%
\begin{tabular}{@{}l*{13}{r}@{}}
    \toprule
    \multirow{2}{*}{\textbf{Method}}
    & \multicolumn{5}{c}{\textbf{Math}}
    & \multicolumn{4}{c}{\textbf{Code}}
    & \multicolumn{3}{c}{\textbf{Instruction Following}}
    & \multirow{2}{*}{\shortstack{\textbf{All}\\\textbf{AVG}}} \\
    \cmidrule(lr){2-6}
    \cmidrule(lr){7-10}
    \cmidrule(lr){11-13}
    & \shortstack{AIME24\\mean@32}
    & \shortstack{AIME25\\mean@32}
    & \shortstack{AMC23\\mean@32}
    & MATH500
    & \textbf{AVG}
    & \shortstack{LCB v6\\pass@1}
    & \shortstack{BCB\\pass@1}
    & \shortstack{HumanEval+\\pass@1}
    & \textbf{AVG}
    & IFBench
    & IFEval
    & \textbf{AVG}
    & \\
    \midrule
    \rowcolor{gray!12}
    \multicolumn{14}{c}{\textit{\textbf{Qwen3-4B}}} \\
    \midrule

    Base
    & 21.46 & 20.31 & 65.55 & 80.20 & 46.88
    & 26.07 & 8.11 & 76.22 & 36.80
    & 26.33 & 79.85 & 53.09
    & 45.59 \\

    Final Specialists
    & 54.69 & 54.48 & 92.27 & 90.40 & 72.96
    & 31.79 & 11.82 & 81.71 & 41.77
    & 40.00 & 83.92 & 61.96
    & 58.90 \\

    \midrule

    Mixed SFT
    & 54.69 & 52.71 & \textbf{92.19} & 89.60 & 72.30
    & \textbf{32.57} & 10.14 & 78.05 & 40.25
    & 38.67 & 83.55 & 61.11
    & 57.89 \\

    Offline MPD
    & 55.52 & \textbf{53.13} & 91.80 & 90.40 & 72.71
    & 32.00 & 9.54 & 77.44 & 39.66
    & 39.33 & 82.62 & 60.98
    & 57.78 \\

    Lightning MOPD
    & 55.21 & 51.46 & 92.11 & \textbf{90.60} & 72.34
    & 31.93 & \textbf{10.98} & 74.39 & 39.10
    & 40.67 & \textbf{85.21} & 62.94
    & 58.13 \\

    Standard MOPD
    & \textbf{55.94} & 52.92 & 91.87 & 90.20 & \textbf{72.73}
    & 32.14 & 9.97 & 79.27 & 40.46
    & 39.67 & 81.89 & 60.78
    & 57.99 \\

    \midrule

    \name{}-SFT
    & 50.63 & 53.02 & 91.64 & 89.40 & 71.17
    & 32.21 & 10.73 & \textbf{81.71} & \textbf{41.55}
    & \textbf{43.00} & 84.10 & \textbf{63.55}
    & \textbf{58.76} \\

    \textbf{\name{}-MPD}
    & \makecell[r]{54.76\\{\scriptsize$\pm 1.49$}}
    & \makecell[r]{52.26\\{\scriptsize$\pm 0.94$}}
    & \makecell[r]{91.88\\{\scriptsize$\pm 0.16$}}
    & \makecell[r]{90.40\\{\scriptsize$\pm 0.20$}}
    & \makecell[r]{72.32\\{\scriptsize$\pm 0.37$}}
    & \makecell[r]{32.12\\{\scriptsize$\pm 0.64$}}
    & \makecell[r]{10.33\\{\scriptsize$\pm 0.27$}}
    & \makecell[r]{79.07\\{\scriptsize$\pm 1.53$}}
    & \makecell[r]{40.51\\{\scriptsize$\pm 0.64$}}
    & \makecell[r]{39.44\\{\scriptsize$\pm 2.36$}}
    & \makecell[r]{83.98\\{\scriptsize$\pm 0.47$}}
    & \makecell[r]{61.71\\{\scriptsize$\pm 1.07$}}
    & \makecell[r]{58.18\\{\scriptsize$\pm 0.50$}} \\

\midrule
\rowcolor{gray!12}
\multicolumn{14}{c}{\textit{\textbf{Qwen3-8B}}} \\
\midrule

Base
& 27.29 & 20.31 & 66.72 & 79.60 & 48.48
& 25.29 & 17.23 & 81.10 & 41.20
& 27.00 & 82.81 & 54.90
& 48.20 \\

Final Specialists
& 62.40 & 54.27 & 92.27 & 88.80 & 74.44
& 34.93 & 17.06 & 82.93 & 44.97
& 45.67 & 88.35 & 67.01
& 62.14 \\

Standard MOPD
& 59.17 & 51.04 & \textbf{91.33} & 87.40 & 72.23
& \textbf{36.07} & \textbf{17.99} & 82.93
& \textbf{45.66}
& \textbf{46.33} & \textbf{87.80} & \textbf{67.07}
& \textbf{61.65} \\

\midrule

\name{}-SFT
& 58.44 & 51.88 & 90.08 & 87.20 & 71.90
& 34.50 & \textbf{17.99} & 80.49 & 44.33
& 44.00 & 87.43 & 65.72
& 60.65 \\

\textbf{\name{}-MPD}
& \makecell[r]{\textbf{59.93}\\{\scriptsize$\pm 0.69$}}
& \makecell[r]{\textbf{52.15}\\{\scriptsize$\pm 0.22$}}
& \makecell[r]{91.12\\{\scriptsize$\pm 0.66$}}
& \makecell[r]{\textbf{88.27}\\{\scriptsize$\pm 0.61$}}
& \makecell[r]{\textbf{72.87}\\{\scriptsize$\pm 0.41$}}
& \makecell[r]{34.90\\{\scriptsize$\pm 0.48$}}
& \makecell[r]{16.53\\{\scriptsize$\pm 0.18$}}
& \makecell[r]{\textbf{83.54}\\{\scriptsize$\pm 2.79$}}
& \makecell[r]{44.99\\{\scriptsize$\pm 0.95$}}
& \makecell[r]{44.00\\{\scriptsize$\pm 0.58$}}
& \makecell[r]{87.49\\{\scriptsize$\pm 0.75$}}
& \makecell[r]{65.75\\{\scriptsize$\pm 0.08$}}
& \makecell[r]{61.20\\{\scriptsize$\pm 0.20$}} \\

\midrule

\rowcolor{gray!12}
\multicolumn{14}{c}{\textit{\textbf{SmolLM3-3B}}} \\
\midrule

Base
& 8.02 & 9.48 & 42.34 & 63.60 & 30.86
& 18.79 & 3.80 & 59.15 & 27.25
& 15.67 & 69.13 & 42.40
& 33.50 \\

Final Specialists
& 36.67 & 27.29 & 76.48 & 83.80 & 56.06
& 25.36 & 5.57 & 67.07 & 32.67
& 46.33 & 82.81 & 64.57
& 51.10 \\

Standard MOPD
& \textbf{37.29} & 26.88 & 73.44
& \textbf{83.20} & \textbf{55.20}
& \textbf{27.00} & \textbf{7.01} & 61.59 & 31.87
& \textbf{44.00} & \textbf{81.33} & \textbf{62.67}
& \textbf{49.91} \\

\midrule

\name{}-SFT
& 36.25 & \textbf{26.98} & 74.45 & 83.00 & 55.17
& 25.50 & 6.00 & 58.54 & 30.01
& 42.00 & \textbf{81.33} & 61.67
& 48.95 \\

\textbf{\name{}-MPD}
& \makecell[r]{36.39\\{\scriptsize$\pm 1.36$}}
& \makecell[r]{26.39\\{\scriptsize$\pm 0.77$}}
& \makecell[r]{\textbf{75.00}\\{\scriptsize$\pm 0.28$}}
& \makecell[r]{81.47\\{\scriptsize$\pm 0.64$}}
& \makecell[r]{54.81\\{\scriptsize$\pm 0.44$}}
& \makecell[r]{26.38\\{\scriptsize$\pm 0.37$}}
& \makecell[r]{6.17\\{\scriptsize$\pm 0.59$}}
& \makecell[r]{\textbf{63.21}\\{\scriptsize$\pm 2.14$}}
& \makecell[r]{\textbf{31.92}\\{\scriptsize$\pm 0.91$}}
& \makecell[r]{41.78\\{\scriptsize$\pm 0.38$}}
& \makecell[r]{81.08\\{\scriptsize$\pm 0.28$}}
& \makecell[r]{61.43\\{\scriptsize$\pm 0.22$}}
& \makecell[r]{49.39\\{\scriptsize$\pm 0.19$}} \\

\bottomrule
\end{tabular}%
}
\end{table*}

Table~\ref{tab:mopd_main} demonstrates that exploration trajectories
support effective consolidation without additional response generation.
On Qwen3-4B, \name{}-MPD achieves $58.18$ overall accuracy,
slightly above Offline MPD ($57.78$) and Lightning MOPD ($58.13$).
The exploration history spans multiple stages
of capability development, potentially providing broader coverage
than either initial-student or final-specialist responses alone.
Its comparable performance to Standard MOPD ($57.99$) also shows that \textit{reuse} approximates \textit{fresh}-based updates.
The additional models exhibit a similar accuracy-efficiency trade-off.
\name{}-MPD reaches $61.20\pm0.20$ on Qwen3-8B and
$49.39\pm0.19$ on SmolLM3-3B, with modest deficits of $0.45$
and $0.52$ percentage points relative to Standard MOPD.
Together with the cost reductions in Section~\ref{subsec:eff},
these results support reuse as a practical alternative to
on-policy consolidation.
Within \name{} variants, cached teacher supervision generally yields
more balanced results than response-token training alone.
SFT attains the highest overall score on Qwen3-4B
($58.76$), as its stronger code and instruction-following results
offset lower math accuracy, notably on AIME24
($50.63$ vs.\ MPD's $54.76$).
MPD further improves overall accuracy over SFT by $0.55$ and
$0.44$ points on Qwen3-8B and SmolLM3-3B, respectively,
highlighting the value of teacher supervision beyond the reused responses themselves.

\subsection{Trajectory Analysis}

\begin{figure*}
    \centering
    \includegraphics[width=\linewidth]{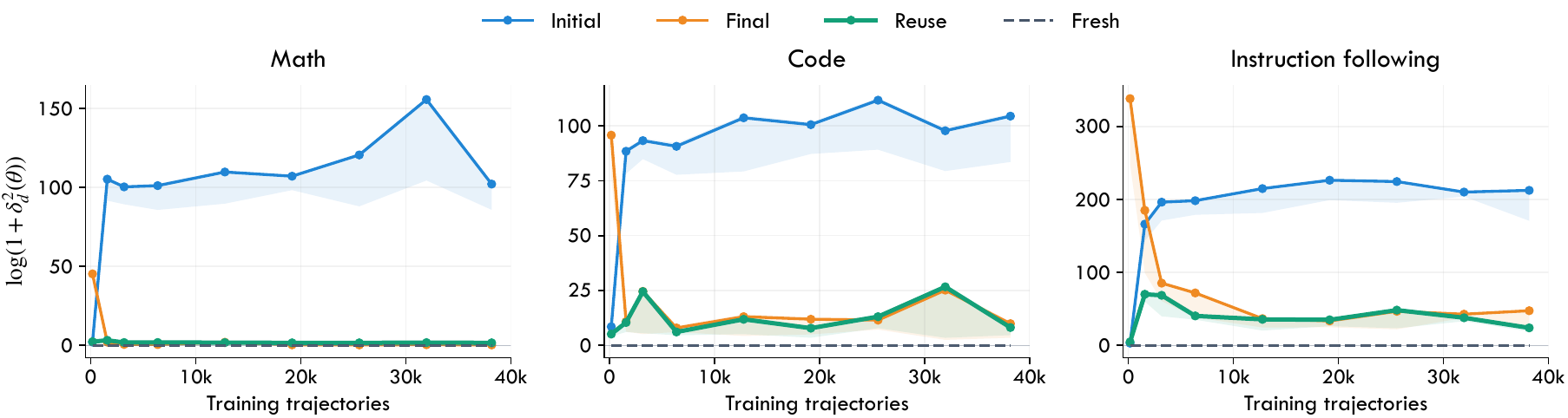}
    \caption{Trajectory-source coverage during Qwen3-4B consolidation. Lower values indicate less distribution mismatch or better coverage.}
    \label{fig:coverage}
\end{figure*}

\textbf{Trajectory coverage comparison.}
To examine how trajectory sources cover the evolving student,
we evaluate their likelihoods on common on-policy responses sampled
throughout consolidation on Qwen3-4B.
We approximate exploration history with a uniform mixture of
policy checkpoints spaced across domain RL, including the shared
initialization and final specialists.
Following the mismatch definition in Eq.~\ref{eq:reuse_mismatch},
we estimate its response-prefix counterpart and report
$\log(1+\delta_d^2(\theta))$.
Figure~\ref{fig:coverage} shows that the initial policy
quickly becomes mismatched as the student evolves, whereas final
specialists exhibit substantial early mismatch.
The historical mixture (\textit{reuse}) moderates these
stage-dependent mismatches, consistent with complementary coverage
across exploration stages, though it does not uniformly outperform
the final specialists.
These diagnostics illustrate coverage dynamics rather than
directly validate the full-trajectory update bound.

\begin{wraptable}{r}{0.4\textwidth}
\centering
\small
\caption{Trajectory-selection ablation.}
\label{tab:selection_summary}
\resizebox{\linewidth}{!}{%
\begin{tabular}{@{}lrrrr@{}}
\toprule
Selector & Math & Code & IF & AVG \\
\midrule
Random     & 71.21 & 40.40 & 61.28 & 57.63 \\
Reward & 71.30 & 41.30 & 60.61 & 57.74 \\
Balance & 72.65 & 41.14 & 62.77 & 58.85 \\
\bottomrule
\end{tabular}%
}
\end{wraptable}

\textbf{Sensitivity to trajectory selection.}
To assess how strongly reuse depends on trajectory selection,
we compare three strategies on Qwen3-4B under MPD settings.
Each retains one of eight responses per prompt: \textit{Random} samples uniformly, \textit{Reward} selects the highest-reward response, and \textit{Balance} samples uniformly from the top four by reward.
Table~\ref{tab:selection_summary} shows relatively modest differences in overall accuracy, with Random already approaching Standard MOPD ($57.63$ vs.\ $57.99$). 
Reward improves only slightly over Random, while Balance achieves the highest overall score in this comparison.
These results suggest that effective reuse does not hinge on careful trajectory selection, while leaving room for improved selectors that better support capability transfer (Appendix~\ref{app:selection_results}).

%% file: Sections/05_relat.tex
\section{Related Work}
\label{sec:related}

\paragraph{Multi-task post-training.}
Common pipelines include sequential or mixed-domain training with
SFT or RL objectives~\citep{luo2023forgetting,lambert2025tulu3,
wang2025nemotroncascade} and training-free model
merging~\citep{yadav2023ties}.
To mitigate conflicts among task objectives, PCGrad~\citep{yu2020pcgrad}
and MGS~\citep{mgs} apply gradient surgery during joint training.
Recent studies further refine training: NPO~\citep{qin2026npo}
uses near-future policy guidance to improve exploration,
CoPD~\citep{copd} interleaves RL and mutual distillation,
D$^3$-MOPD~\citep{sun2026d3mopd} adapts domain sampling to learning
progress during consolidation, and Open-MOPD~\citep{openmopd}
rebalances domain supervision to address capability imbalance.
These refinements are complementary to \name{}: they modify how
consolidation learns, whereas \name{} changes the source of
trajectories and teacher supervision.
% We focus on RL-then-MOPD, where independently trained specialists
% supervise a shared student using fresh rollouts~\citep{deepseekv4,mopd}.
% Unlike existing approaches that leave domain-RL trajectories unused,
% \name{} reuses both the specialist policies and historical trajectories
% produced during domain-wise exploration.

\paragraph{Efficient policy distillation and trajectory reuse.}
Offline sequence-level distillation learns from teacher-generated
responses~\citep{seqkd}, while Lightning OPD~\citep{wu2026lightningopd}
uses fixed responses from the SFT-initialized student.
These sources correspond to the two endpoints of domain exploration:
the final specialist policy and the initial policy.
AsyncOPD~\citep{asyncopd} enables asynchronous distillation using
stale student rollouts.
At finer granularity, ReOPD~\citep{reopd} reuses teacher-generated
prefixes while generating new student actions at selected steps.
With fixed responses and frozen teachers, both endpoint sources
and historical reuse permit offline pre-scoring for
online-teacher-free student training.
\name{} additionally eliminates response generation for consolidation
by reusing trajectories already collected during domain RL.
By combining complementary behaviors across exploration substages,
its historical policy mixture may cover the evolving student better
than either endpoint alone.

%% file: Sections/06_conclu.tex
\section{Conclusion}
\label{sec:conclu}

Conventional RL-then-MOPD incurs substantial system costs from
redundant response generation and persistent teacher GPU residency.
We formulate this pipeline as domain-wise exploration followed by
shared exploitation and introduce \name{} to reuse both outputs of
exploration: final specialist policies and historical trajectories.
By combining trajectory reuse with teacher pre-scoring, \name{}
enables rollout-free consolidation and online-teacher-free student
training.
Our analysis relates the discrepancy from on-policy distillation
updates to distribution mismatch and characterizes how historical
policy mixtures can improve coverage.
Experiments across three models show substantial reductions in GPU-hours and training time, while lowering peak GPU demand and maintaining comparable multi-domain performance.

%% file: Sections/08_implementation_appendix.tex
\section{Implementation Details}
\label{app:training_details}

\subsection{Student Objective}
\label{app:student_objective}
For each mixed-domain minibatch, let $\mathcal C_d$ be the subset
sampled from $\widetilde{\mathcal B}_d$. For each observed token,
define
\begin{align}
    A_t(\tau)&=\operatorname{sg}\left[
    \operatorname{clip}\bigl(q_{d,t}^T(\tau)
    -\log\pi_\theta(y_t\mid h_t),-c,c\bigr)\right],
    \label{eq:distillation_advantage}\\
    \rho_t(\tau)&=
    \frac{\pi_\theta(y_t\mid h_t)}{\bar\pi(y_t\mid h_t)},
    \label{eq:student_ratio}
\end{align}
where $\operatorname{sg}$ denotes stop-gradient, $c>0$ bounds the
distillation signal, and $\bar\pi$ is the student checkpoint anchoring
the minibatch updates. We optimize
\begin{equation}
    \mathcal L_{\mathrm{PD}}(\theta)=
    -\frac{1}{N_{\mathcal C}}\sum_d
    \sum_{(\tau,\mathbf q^T)\in\mathcal C_d}\sum_t
    \min\!\left\{\rho_t A_t,
    \operatorname{clip}(\rho_t,1-\epsilon,1+\epsilon)A_t\right\},
    \label{eq:shared_distillation}
\end{equation}
where $N_{\mathcal C}$ is the total number of valid response tokens
in the combined minibatch and $\epsilon>0$ controls ratio clipping.
Student and anchor log-probabilities are evaluated on the reused
responses without generating new ones.

\subsection{Training Configurations}
\label{app:consolidation_protocol}

\paragraph{Models and initialization.}
We use the released post-trained checkpoints of
Qwen3-4B\footnote{\url{https://huggingface.co/Qwen/Qwen3-4B}},
Qwen3-8B\footnote{\url{https://huggingface.co/Qwen/Qwen3-8B}}, and
SmolLM3-3B\footnote{\url{https://huggingface.co/HuggingFaceTB/SmolLM3-3B}},
all in non-thinking mode. Qwen3-4B hosts the full baseline comparison;
the other models test generalization across sizes and families.
Domain specialists and the shared student are independently
initialized from the same checkpoint, denoted \emph{Base} in the
result tables. PD methods use the fixed final specialists as teachers.

\paragraph{Implementation and evaluation.}
All methods are implemented in verl with vLLM and run on two nodes
with eight NVIDIA H100 GPUs each; the efficiency tables report
per-method resource allocations. We evaluate final student
checkpoints.

\paragraph{Random seeds.}
The default training seed is 42. The \name{}-MPD accuracy rows in
Table~\ref{tab:mopd_main} report the mean and sample standard
deviation over seeds 42, 43, and 44, with RL trajectories and
final-teacher scores held fixed across seeds.
Other rows report single-seed results unless noted otherwise. Each seed's domain and
aggregate scores are computed before averaging across seeds.

\paragraph{Baseline construction.}
Lightning MOPD adapts the fixed initial-student trajectory source of
Lightning OPD to multi-domain PD; it does not reproduce the full
SFT-OPD pipeline. Offline MPD, Lightning MOPD, and \name{}-MPD cache
final-specialist scores before student training. Mixed SFT and
\name{}-SFT use final-specialist response tokens and historical RL
response tokens, respectively, without teacher scoring during training.

\paragraph{Trajectory selection.}
\textit{Random} samples one of the eight RL responses per prompt
occurrence; \textit{Reward} selects the response with the highest
recorded RL reward; \textit{Balance} samples uniformly among
the four highest-reward responses.
Table~\ref{tab:mopd_main} uses \textit{Balance} on Qwen3-4B and
\textit{Random} on Qwen3-8B and SmolLM3-3B for both \name{} variants.

\paragraph{Training configurations.}
Tables~\ref{tab:rl_training_config} and~\ref{tab:opd_training_config}
summarize the key RL and distillation settings.
Both stages use 64 prompts per batch, a learning rate of
$1\times10^{-6}$, a PPO mini-batch of 16, sampling at temperature
$1.0$ with top-$p=1$, a maximum response length of 10,240 tokens,
token-mean loss aggregation, and seed 42.
Each domain uses 17K training examples from
DAPO-Math-17k~\citep{dapo2025},
DeepCoder-Preview-Dataset~\citep{deepcoder2025}, or
Nemotron-RL-instruction\_following~\citep{nvidia2025nemotronif}.
Math uses a final-answer verifier; code uses an all-tests-pass reward
with at most 15 tests, a 10-second timeout, and a 1,024 MB memory
limit; instruction following uses the fraction of verified
constraints satisfied.

\begin{table}[H]
\centering
\small
\begin{minipage}[t]{0.48\linewidth}
\centering
\caption{Domain-RL settings.}
\label{tab:rl_training_config}
\setlength{\tabcolsep}{4pt}
\begin{tabular}{@{}ll@{}}
\toprule
Setting & Value \\
\midrule
Objective & GRPO \\
Training batches & 200 \\
Responses per prompt & 8 \\
PPO clipping (low / high) & 0.2 / 0.28 \\
KL reward & Disabled \\
Overlong buffer & 4,096 tokens \\
\bottomrule
\end{tabular}
\end{minipage}\hfill
\begin{minipage}[t]{0.48\linewidth}
\centering
\caption{Policy-distillation settings.}
\label{tab:opd_training_config}
\setlength{\tabcolsep}{4pt}
\begin{tabular}{@{}ll@{}}
\toprule
Setting & Value \\
\midrule
Objective & PPO-style clipped PD \\
Responses per prompt & 1 \\
Prompt proportions & $1{:}1{:}1$ \\
Teacher scoring & Temperature $=1.0$ \\
Distillation signal & \makecell{Teacher$-$student\\log-prob.\ difference} \\
\bottomrule
\end{tabular}
\end{minipage}
\end{table}

\subsection{Evaluation Configurations}
\label{app:benchmark_protocol}

For math, we report mean@32 (averaging correctness over 32 responses
per problem) on AIME 2024/2025 and AMC 2023, and single-response
accuracy on MATH500~\citep{hendrycks2021math,lightman2023verify}.
For code, we report pass@1 with eight samples per problem on the
175-problem LiveCodeBench v6~\citep{jain2024livecodebench}
\texttt{test6} generation subset and the
148-problem BigCodeBench~\citep{zhuo2024bigcodebench} Complete-Hard
subset, and pass@1 with a single sample on HumanEval+~\citep{liu2023evalplus}.
IFBench~\citep{pyatkin2025ifbench} and IFEval~\citep{zhou2023ifeval}
use strict prompt-level accuracy.
Math evaluation uses temperature $0.6$ and top-$p=0.95$; code
and instruction-following evaluation use temperature $0.7$,
top-$p=0.8$, top-$k=20$, and seed 1234, with an additional presence
penalty of $1.5$ for code.

%% file: Sections/09_theory_appendix.tex
\section{Theoretical Details}
\label{app:reuse_theory}

We fix the student state $\theta$, final teachers, and historical
policies. Write $P_d(x,y)=\mathcal D_d(x)p_{\mathrm{fresh},d}(y\mid x;\theta)$
and $Q_d(x,y)=\mathcal D_d(x)p_{\mathrm{reuse},d}(y\mid x)$.
The common prompt marginal gives
$\chi^2(P_d\|Q_d)=\delta_d^2(\theta)$, where
\begin{equation}
    \chi^2(P\|Q)=
    \mathbb E_Q\left[\left(\frac{\mathrm dP}{\mathrm dQ}-1\right)^2\right],
    \qquad P\ll Q.
    \label{eq:chi_square_definition}
\end{equation}
These are population distributions, rather than realized finite buffers.

\subsection{Proof of Theorem~\ref{thm:reuse_discrepancy}}
\label{app:discrepancy_proof}

\begin{proof}
Let $W_d=\mathrm dP_d/\mathrm dQ_d$ and
$g_{b,d}=\mathbb E_{\mathcal D_d,p_{b,d}}f_d$.
Changing measure and applying Cauchy-Schwarz yields
\begin{align}
    g_{\mathrm{fresh},d}-g_{\mathrm{reuse},d}
    &=\mathbb E_{Q_d}[(W_d-1)f_d],\\
    \|g_{\mathrm{fresh},d}-g_{\mathrm{reuse},d}\|_2
    &\le
    \sqrt{\mathbb E_{Q_d}(W_d-1)^2}
    \sqrt{\mathbb E_{Q_d}\|f_d\|_2^2}
    \le\delta_d B_d.
\end{align}
The triangle inequality over the domain-weighted sum gives
Eq.~\ref{eq:reuse_discrepancy_bound}.
\end{proof}

Absolute continuity holds if a retained policy with positive mixture
weight assigns positive probability to every response allowed by the
fresh source. Hard sampling filters can violate this condition.
An admissible moment bound is $B_d=cL_{\max}G$ when response lengths
are bounded by $L_{\max}$ and
$\|\nabla_\theta\log\pi_\theta(y_t\mid h_t)\|_2\le G$.
Clipping the advantage alone does not bound the score function or
distribution mismatch. The ratio $W_d$ is a proof device, not an
importance weight used by the algorithm.

\subsection{Proof of Proposition~\ref{prop:history_coverage}}
\label{app:coverage_proof}

\begin{proof}
Fix $d,x$ and a retained component with $\alpha_{d,i}>0$.
Abbreviate $p=\pi_\theta(\cdot\mid x)$,
$q=p_{\mathrm{reuse},d}(\cdot\mid x)$, and
$q_i=\pi_{d,i}(\cdot\mid x)$.
Since $q\ge\alpha_{d,i}q_i$ pointwise,
\begin{equation}
    1+\chi^2(p\|q)
    =\sum_y\frac{p(y)^2}{q(y)}
    \le\frac{1}{\alpha_{d,i}}\sum_y\frac{p(y)^2}{q_i(y)}
    =\frac{1+\chi^2(p\|q_i)}{\alpha_{d,i}}.
\end{equation}
Zero-over-zero terms are zero; if the right-hand divergence is
infinite, the bound is vacuous. Averaging over
$x\sim\mathcal D_d$ proves the proposition.
\end{proof}

If initial-policy responses are retained with weight $\alpha_{d,0}>0$,
the bound at $\pi_{\theta_0}=\pi_0$ becomes
$\delta_d^2(\theta_0)\le\alpha_{d,0}^{-1}-1$.
This gives an initial coverage component, not control over later drift.

\subsection{An Example of Complementary Historical Coverage}
\label{app:coverage_example}

For one prompt and $0<\varepsilon<1/3$, consider three responses:
\begin{center}
\begin{tabular}{lccc}
\toprule
Distribution & $A$ & $B$ & $C$ \\
\midrule
Initial & $1-2\varepsilon$ & $\varepsilon$ & $\varepsilon$ \\
Final & $\varepsilon$ & $1-2\varepsilon$ & $\varepsilon$ \\
Intermediate & $\varepsilon$ & $\varepsilon$ & $1-2\varepsilon$ \\
Current student & $1/3$ & $1/3$ & $1/3$ \\
\bottomrule
\end{tabular}
\end{center}
The uniform historical mixture equals the current student $p$, so
$\chi^2(p\|q_{\mathrm{reuse}})=0$.
Each endpoint instead has mismatch
\[
    \chi^2(p\|q_{\mathrm{initial}})
    =\chi^2(p\|q_{\mathrm{final}})
    =\frac{1}{9}\left(\frac{1}{1-2\varepsilon}
    +\frac{2}{\varepsilon}\right)-1>0.
\]
Even the best endpoint-only mixture retains positive mismatch:
\[
    \min_{\lambda\in[0,1]}
    \chi^2\!\left(p\middle\|
    \lambda q_{\mathrm{initial}}+(1-\lambda)q_{\mathrm{final}}\right)
    =\frac{(1-3\varepsilon)^2}{9\varepsilon(1-\varepsilon)}>0.
\]
Both endpoints underrepresent $C$; the intermediate policy fills
this gap. Conversely, an endpoint has zero mismatch when the student
equals that endpoint. The example establishes a possible benefit,
not a universal ranking or an assumption that students follow
historical mixtures.

\subsection{Connection to the Implemented Update}
\label{app:implementation_connection}

\paragraph{PPO anchor and inner updates.}
The implemented objective is Eq.~\ref{eq:shared_distillation}.
At $\pi_\theta=\bar\pi$, all token ratios equal one, ratio clipping
is inactive, and the negative gradient is the sum of the
$f_d$ contributions divided by the valid token count.
With the sampling distribution and detached advantages held fixed,
$g_b$ is also the gradient at $\vartheta=\theta$ of
\[
    \mathcal S_b(\vartheta;\theta)
    =\sum_d w_d\mathbb E_{\mathcal D_d,p_{b,d}}
    \left[\sum_t A_{d,t}(\theta)
    \log\pi_\vartheta(y_t\mid h_t)\right].
\]
Thus the comparison concerns raw local surrogate gradients, not
task-reward gradients or optimizer-preconditioned steps.
Away from the anchor, token contributions acquire the PPO ratio
and clipping indicator. The same change-of-measure argument applies
to that common contribution function if its second moment is finite.
Note that a fresh minibatch is still sampled from the anchor
$\bar\pi$ rather than the updated $\pi_\theta$; the anchor ratio
therefore corrects for student updates within the minibatch window,
not for the mismatch between the student and the historical
explorer that generated a reused response.

\paragraph{Global token normalization.}
Let $P_b(d,x,y)=w_d\mathcal D_d(x)p_{b,d}(y\mid x)$,
$\ell=|y|$, and $Z_b=\mathbb E_{P_b}\ell\in(0,\infty)$.
Population token-mean updates are $G_b=g_b/Z_b$.
For $W=\mathrm dP_{\mathrm{fresh}}/\mathrm dP_{\mathrm{reuse}}$
and $f(d,x,y)=f_d(x,y)$,
\begin{equation}
    G_{\mathrm{fresh}}-G_{\mathrm{reuse}}
    =\frac{\mathbb E_{P_{\mathrm{reuse}}}
    [(W-1)(f-\ell G_{\mathrm{reuse}})]}{Z_{\mathrm{fresh}}}.
    \label{eq:normalized_update_difference}
\end{equation}
This follows from
$\mathbb E_{P_{\mathrm{reuse}}}(f-\ell G_{\mathrm{reuse}})=0$.
Cauchy-Schwarz then gives
\[
    \|G_{\mathrm{fresh}}-G_{\mathrm{reuse}}\|_2
    \le
    \frac{\sqrt{\sum_d w_d\delta_d^2}\,
    \sqrt{\mathbb E_{P_{\mathrm{reuse}}}
    \|f-\ell G_{\mathrm{reuse}}\|_2^2}}{Z_{\mathrm{fresh}}},
\]
when these moments are finite. Different response lengths therefore
affect normalized update magnitudes. A finite minibatch has a random
token-count denominator; its expected update need not equal the
population ratio above.

\paragraph{Finite buffers and selection.}
\label{app:reuse_scope}
If $\widehat g_{\mathrm{reuse}}$ is the expected unnormalized update
under sampling from a realized buffer, then
\[
    \|\widehat g_{\mathrm{reuse}}-g_{\mathrm{fresh}}\|_2
    \le\Delta(\theta)
    +\|\widehat g_{\mathrm{reuse}}-g_{\mathrm{reuse}}\|_2.
\]
The additional finite-buffer error, minibatch variation, and adaptive
student evolution are not bounded by the population theorem.
Content-dependent selection requires using the resulting replay
distribution; differing prompt marginals require joint
prompt-response distributions. The results describe local update
discrepancy, not convergence or final task accuracy.

%% file: Sections/10_results_appendix.tex
\section{Training Dynamics and Additional Results}
\label{app:additional_results}

\subsection{Training Dynamics}
\label{app:dynamics}

\begin{figure}[htbp]
    \centering
    \begin{subfigure}{\linewidth}
        \centering
        \includegraphics[width=\linewidth]{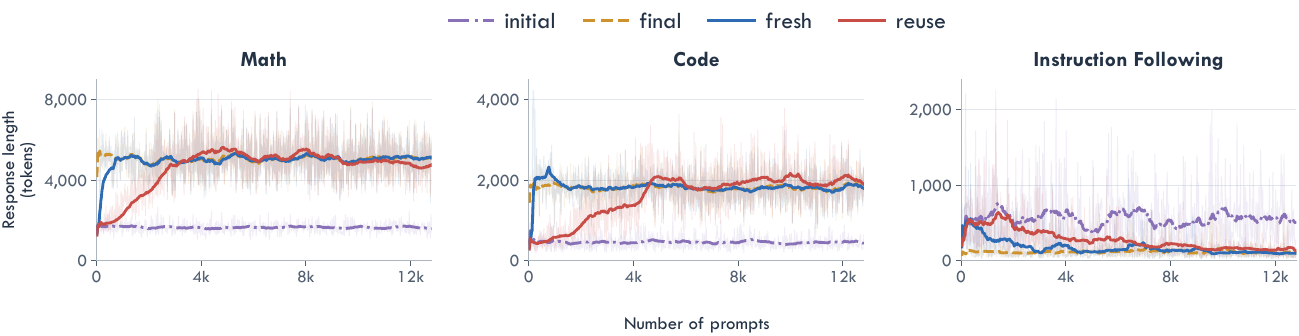}
        \caption{Response length by trajectory source.
        Mixed SFT and Offline MPD share the same \textit{final}
        responses.}
        \label{fig:response_length_dynamics}
    \end{subfigure}

    \begin{subfigure}{\linewidth}
        \centering
        \includegraphics[width=0.96\linewidth]{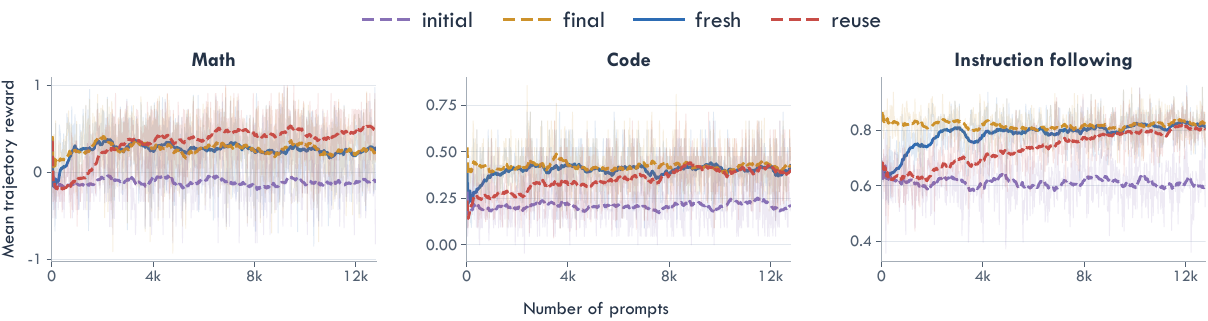}
        \caption{Reward by trajectory source.}
        \label{fig:trajectory_reward}
    \end{subfigure}

    \begin{subfigure}{\linewidth}
        \centering
        \includegraphics[width=0.93\linewidth]{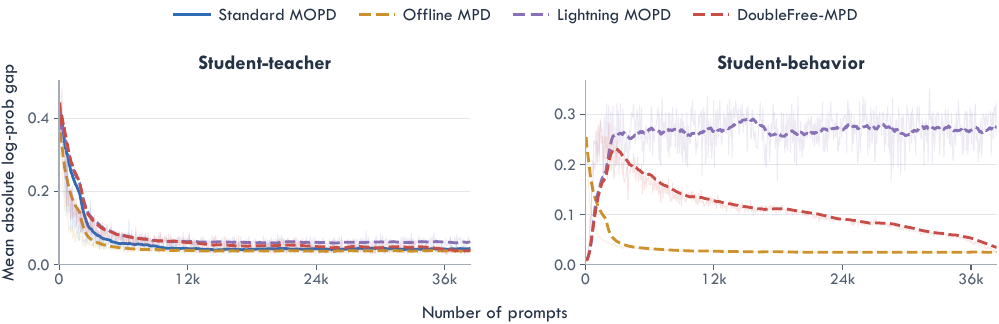}
        \caption{Student-teacher and student-behavior absolute
        sampled-token log-probability gaps.
        Standard MOPD is omitted from the behavior gap, as its
        responses are generated by the pre-update student itself.}
        \label{fig:trajectory_logprob_gaps}
    \end{subfigure}
    \caption{Training dynamics by trajectory source on Qwen3-4B.}
    \label{fig:training_dynamics}
\end{figure}

Figure~\ref{fig:training_dynamics} compares response sources
during Qwen3-4B consolidation.
Figures~\ref{fig:response_length_dynamics}
and~\ref{fig:trajectory_reward} analyze the trajectories themselves
through their lengths and rewards, while
Figure~\ref{fig:trajectory_logprob_gaps} tracks training dynamics
through the student's log-probability gaps.
Throughout, \textit{fresh} denotes current-student
responses from Standard MOPD, \textit{reuse} the historical RL
responses consumed by \name{}-MPD, \textit{initial} the
initial-student responses from Lightning MOPD, and \textit{final}
the final-specialist responses shared by Mixed SFT and Offline MPD.
Cached-source curves follow the consumption order of fixed responses;
\textit{fresh} instead follows the evolving student's online responses.
The \textit{reuse} trace uses \textit{Random} selection to
exclude reward-based selection effects.

The \textit{final} responses average $5{,}061$, $1{,}820$, and
$120$ tokens in math, code, and instruction following, respectively,
versus \textit{initial}'s $1{,}660$, $464$, and $556$.
The longer math and code responses dominate the total token volume,
so generating from final specialists is correspondingly more expensive.
Lightning MOPD instead generates from the weaker initial student,
whose responses are substantially shorter, so its additional
generation is correspondingly faster.
\name{} reuses responses already produced during domain RL and
incurs no additional preparation generation.
Weighting the distillation loss to balance the per-domain
token-count differences observed above is a natural direction for
future work.

Figure~\ref{fig:trajectory_reward} shows that \textit{reuse}
responses attain higher mean reward than \textit{initial} in all
three domains, while remaining below \textit{fresh} and
\textit{final} in code and instruction following.
Since \textit{reuse} mixes responses from policies across RL
substages, it also naturally spans a wider reward range than any
single-policy source, providing more diverse responses for the
student to learn from.

In Figure~\ref{fig:trajectory_logprob_gaps}, \name{}-MPD's
student-teacher gap steadily decreases and approaches Offline
MPD's late-stage level, indicating that reused trajectories move
the student toward the teachers about as effectively as
teacher-generated responses do.
The teacher gap is logged during actor updates, and the behavior
gap compares the pre-update student with the policy that generated
each cached response.
Each method is logged on its own responses.
Meanwhile, its student-behavior gap also
\emph{decreases} as training consumes later RL trajectories,
rather than growing as a fixed offline source would suggest.
The consumption order explains this: consolidation replays the
exploration progression from early, initial-like policies toward
the final specialist, so the behavior distribution advances in the
same direction as the student itself.

\subsection{Sensitivity to Trajectory Selection}
\label{app:selection_results}

We ablate the trajectory selector $\phi$ in
Eq.~\ref{eq:support_selection} for both \name{} variants on
Qwen3-4B; each selector retains one of eight RL responses per
prompt occurrence.
Table~\ref{tab:reuse_sft_selection_ablation} shows that overall
accuracy varies modestly across selectors: \textit{Random} already
approaches the reference method in each block, \textit{Reward}
brings small gains, and \textit{Balance} performs best for both
variants.
Effective reuse therefore does not hinge on careful response
selection, while better selectors still leave room for further
gains.

\begin{table*}[htbp]
\centering
\caption{Trajectory-selection ablations for \name{}-SFT and
\name{}-MPD on Qwen3-4B (\%).
Bold indicates the best score in each column within each block.}
\label{tab:reuse_sft_selection_ablation}

\footnotesize
\setlength{\tabcolsep}{3.5pt}
\renewcommand{\arraystretch}{1.15}

\resizebox{\linewidth}{!}{%
\begin{tabular}{@{}l*{13}{r}@{}}
    \toprule
    \multirow{2}{*}{\textbf{Method}}
    & \multicolumn{5}{c}{\textbf{Math}}
    & \multicolumn{4}{c}{\textbf{Code}}
    & \multicolumn{3}{c}{\textbf{Instruction Following}}
    & \multirow{2}{*}{\shortstack{\textbf{All}\\\textbf{AVG}}} \\
    \cmidrule(lr){2-6}
    \cmidrule(lr){7-10}
    \cmidrule(lr){11-13}
    & \shortstack{AIME24\\mean@32}
    & \shortstack{AIME25\\mean@32}
    & \shortstack{AMC23\\mean@32}
    & MATH500
    & \textbf{AVG}
    & \shortstack{LCB v6\\pass@1}
    & \shortstack{BCB\\pass@1}
    & \shortstack{HumanEval+\\pass@1}
    & \textbf{AVG}
    & IFBench
    & IFEval
    & \textbf{AVG}
    & \\
    \midrule

    \rowcolor{gray!12}
    \multicolumn{14}{c}{\textit{\textbf{\name{}-SFT}}} \\
    \midrule

    Mixed SFT
    & \textbf{54.69} & 52.71 & \textbf{92.19}
    & 89.60 & \textbf{72.30}
    & \textbf{32.57} & 10.14 & 78.05 & 40.25
    & 38.67 & 83.55 & 61.11
    & 57.89 \\

    \midrule

    Random
    & 52.71 & 51.04 & 90.39 & \textbf{90.40} & 71.14
    & 31.93 & \textbf{10.90} & 80.49 & 41.10
    & 39.67 & 83.92 & 61.79
    & 58.01 \\

    Reward
    & 53.44 & \textbf{53.33} & 91.88 & 88.60 & 71.81
    & 31.50 & 10.56 & \textbf{81.71} & 41.25
    & 41.33 & 83.92 & 62.63
    & 58.56 \\

    Balance
    & 50.63 & 53.02 & 91.64 & 89.40 & 71.17
    & 32.21 & 10.73 & \textbf{81.71} & \textbf{41.55}
    & \textbf{43.00} & \textbf{84.10} & \textbf{63.55}
    & \textbf{58.76} \\

    \midrule

    \rowcolor{gray!12}
    \multicolumn{14}{c}{\textit{\textbf{\name{}-MPD}}} \\
    \midrule

    Standard MOPD
    & 55.94 & \textbf{52.92} & 91.87 & \textbf{90.20}
    & \textbf{72.73}
    & 32.14 & 9.97 & 79.27 & 40.46
    & 39.67 & 81.89 & 60.78
    & 57.99 \\

    \midrule

    Random
    & 53.85 & 50.31 & \textbf{91.88} & 88.80 & 71.21
    & 32.07 & 10.47 & 78.66 & 40.40
    & 39.00 & \textbf{83.55} & 61.28
    & 57.63 \\

    Reward
    & 54.38 & 50.21 & 91.41 & 89.20 & 71.30
    & 32.00 & \textbf{10.81} & \textbf{81.10} & \textbf{41.30}
    & 37.67 & \textbf{83.55} & 60.61
    & 57.74 \\

    Balance
    & \textbf{56.04} & 52.50 & \textbf{91.88}
    & \textbf{90.20} & 72.65
    & \textbf{32.79} & 10.14 & 80.49 & 41.14
    & \textbf{42.00} & \textbf{83.55} & \textbf{62.77}
    & \textbf{58.85} \\

    \bottomrule
\end{tabular}%
}

\end{table*}